\documentclass{article}

\usepackage[in]{fullpage}  
\usepackage{amsmath,amsfonts,bm}

\def\eqref#1{equation~\ref{#1}}
\def\1{\bm{1}}

\DeclareMathAlphabet{\mathsfit}{\encodingdefault}{\sfdefault}{m}{sl}
\SetMathAlphabet{\mathsfit}{bold}{\encodingdefault}{\sfdefault}{bx}{n}

\DeclareMathOperator*{\argmin}{argmin}

\DeclareMathOperator{\diag}{diag}

\DeclareMathOperator{\sign}{sign}

\usepackage[authoryear,round]{natbib} 
\usepackage{hyperref}
\usepackage{url}
\usepackage{booktabs} 
\usepackage{lineno,refcount}
\usepackage{wrapfig}
\usepackage{graphicx}
\usepackage{dblfloatfix}
\usepackage{amsmath}
\usepackage{amsfonts}
\usepackage{soul}
\usepackage[format=plain, justification=justified]{subcaption}
\usepackage[format=plain, justification=justified,singlelinecheck=off]{caption}
\usepackage{resizegather}
\usepackage{tikz}
\usepackage{tikzscale}
\usepackage{float}
\usepackage{siunitx}
\usepackage{multirow}
\graphicspath{{Figures/}} 
\usepackage{etoolbox}
\usepackage{enumitem}
\usepackage{enumerate}
\usepackage[normalem]{ulem}
\usepackage{mathtools}
\usepackage{amsthm}
\usepackage{ragged2e}
\usepackage[letterspace=-50]{microtype}
\usepackage{sidecap}
\usepackage[bottom]{footmisc}
\usepackage{color, colortbl}

\newcommand{\pub}{\mathrm{tar}}
\newcommand{\pri}{\mathrm{hid}}

\newtheorem{definition}{Definition}
\newtheorem{lemma}{Lemma}
\newtheorem{theorem}{Theorem}

\sidecaptionvpos{figure}{c}

\title{Unsupervised Information Obfuscation for Split Inference of \\Neural Networks\vspace{.2cm}}

\author{
  Mohammad Samragh\thanks{Work was done while interning at Qualcomm AI Research.} \\
  ECE Department,~UC San Diego\\
  \texttt{msamragh@ucsd.edu}
   \and
 Hossein Hosseini\thanks{Qualcomm AI Research is an initiative of Qualcomm Technologies, Inc.} \\
  Qualcomm AI Research\\
  \texttt{hhossein@qti.qualcomm.com}\vspace{0.15cm}
  \and
 Aleksei Triastcyn\footnotemark[2] \\
  Qualcomm AI Research\\
  \texttt{atriastc@qti.qualcomm.com}
  \and
 Kambiz Azarian\footnotemark[2] \\
  Qualcomm AI Research\\
  \texttt{kambiza@qti.qualcomm.com}\vspace{0.15cm}
  \and
 Joseph Soriaga\footnotemark[2] \\
  Qualcomm AI Research\\
  \texttt{jsoriaga@qti.qualcomm.com}
  \and
 Farinaz Koushanfar \\
  ECE Department,~UC San Diego\\
  \texttt{farinaz@ucsd.edu}
}

\date{}
\begin{document}
\maketitle

\begin{abstract}
Splitting network computations between the edge device and a server enables low edge-compute inference of neural networks but might expose sensitive information about the test query to the server. To address this problem, existing techniques train the model to minimize information leakage for a given set of sensitive attributes. In practice, however, the test queries might contain attributes that are not foreseen during training. 

We propose instead an unsupervised obfuscation method to discard the information irrelevant to the main task. We formulate the problem via an information theoretical framework and derive an analytical solution for a given distortion to the model output. In our method, the edge device runs the model up to a split layer determined based on its computational capacity. 
% It then obfuscates the obtained feature vector by removing the components in the null space of the next layer of the model as well as the low-energy components of the remaining signal. 
It then obfuscates the obtained feature vector based on the first layer of the server's model by removing the components in the null space as well as the low-energy components of the remaining signal. 
Our experimental results show that our method outperforms existing techniques in removing the information of the irrelevant attributes and maintaining the accuracy on the target label. We also show that our method reduces the communication cost and incurs only a small computational overhead.

\end{abstract}

% keywords can be removed
% \keywords{DNN Split Inference \and Data Obfuscation \and Cloud-Server Inference}

\section{Introduction}
% \vspace{-0.3cm}

In recent years, the surge in cloud computing and machine learning has led to the emergence of Machine Learning as a Service (MLaaS), where the compute capacity of the cloud is used to analyze the data generated on edge devices. One shortcoming of the MLaaS framework is the leakage of the clients' privacy-sensitive data to the cloud server. To address this problem, several cryptography-based solutions have been proposed which provide provable security at the cost of increasing the communication cost and delay of the cloud inference by orders of magnitude~\citep{juvekar2018gazelle,rathee2020cryptflow2}. Such cryptography-based solutions are applicable in use-cases where delay is tolerable such as healthcare~\citep{microsoftAIHealth}, but not in scenarios where millions of clients request fast and low communication cost responses such as in Amazon Alexa or Apple Siri applications.
% however, the large communication and computational overhead of cryptographic tools hinders their wide-scale adoption in resource-constrained edge devices. 
A light-weight alternative to cryptographic solutions is to hide sensitive information on the edge device, e.g., by blurring images before sending them to the server~\citep{vishwamitra2017blur}. This approach, however, is task-specific and is not viable for generic applications.

Another approach is split inference which provides a generic and computationally efficient data obfuscation framework~\citep{kang2017neurosurgeon,chi2018privacy}. 
In this approach, the service provider trains the model and splits it into two sub-models, $M_c$ and $M_s$, where $M_c$ contains the first few layers of the model and $M_s$ contains the rest. 
% The server makes an agreement with the client about the number of layers to be run locally, then splits the model accordingly. The server shares the fist few layers of the network weights with clients. 
At inference time, the client runs $M_c$ on the edge device and sends the resulting feature vector $z=M_c(x)$ to the server, which then computes the target label as $\hat{y}^{\pub}=M_s(z)$. To protect the sensitive content of the client's query, the model is required to be designed such that $z$ only contains the information related to the underlying task. This aligns well with the recent privacy laws, such as GDPR~\citep{gdpr}, that restrict the amount of collected information to the necessary minimum. For instance, when sending facial features for cell-phone authentication, the client does not want to disclose other information such as their mood or their makeup. We denote such hidden attributes as $y^\pri$ in the remainder of this paper. 

\begin{figure}[t]
\centering
\includegraphics[width=0.5\columnwidth]{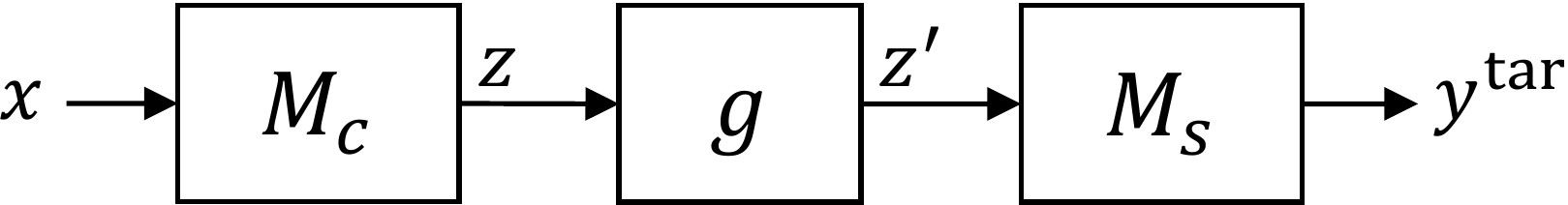}
\caption{\small Unsupervised data obfuscation in split inference setting. 
$M_c$ and $M_s$ are the client and server models, respectively, and $g$ is the obfuscation function. The client computes $z'=g(M_c(x))$ and sends $z'$ to the server to predict the target attribute as $\hat{y}^{\pub}=M_s(z')$. The obfuscator is designed to minimize the content in $z$ that is irrelevant to $y^\pub$, and also to be efficient enough to run on the edge device.
% , while maintaining the utility, i.e., the server should be able to infer the target attribute from $z'$ accurately.
}
\label{fig:scenario}
\end{figure}

% Split-inference is not readily applicable to a pre-trained neural network since the hidden layer features may contain private information. For instance, a network trained for gender classification extracts features that also contain information about race or identity~(\citep{song2019overlearning}). 
% As seen in Figure~\ref{fig:scenario}, information leakage can be utilised by an adversary that trains the model $M_a$ to extract the hidden label $y^{\pri}$ from feature vector $z$.

%\edit{what should we call public and private attributes? (public and secret) or (main and auxilary) or (intended and unintended)}

% \mohammad{The information leakage in split inference is modeled by an adversary with trained model $M_a$, which extracts hidden attribute $y^\pri$ from the feature vector.}

% an adversary tries to extract information that were not originally intended to be leaked. The adversary is modeled by  a function $M_a$ that extracts hidden attribute $y^\pri$ from the feature vector.}
Current methods for data obfuscation in split inference aim to remove the information corresponding to a known list of hidden attributes. For example, adversarial training \citep{feutry2018learning} and noise injection \citep{mireshghallah2020shredder} methods minimize the accuracy of an adversary model $M_a(z)$ on $y^\pri$, and the information bottleneck method \citep{osia2018deep} trains the model to minimize the mutual information between the query $z$ and $y^{\pri}$. 
% \citep{mireshghallah2020shredder} propose to add structured noise to the query $z$ and send the noisy features to the server. 
% and create $\widetilde{z}$ such that $M_s(\widetilde{z})$ is an accurate estimation of the public label $y^{\pub}$ but  $M_a(\widetilde{z})$ does not accurately describe $y^{\pri}$. 
% The above works commonly assume knowledge of private labels at training time. 
The set of hidden attributes, however, can vary from one query to another. Hence, it is not feasible to foresee all types of attributes that could be considered sensitive for a specific MLaaS application. Moreover, the need to annotate inputs with all possible hidden attributes significantly increases the training cost as well.

% % \begin{SCfigure}[50][t]
% \begin{figure}
%     \centering
%     \includegraphics[width=0.75\columnwidth]{Scenario.png}
%     \caption{\small Split inference setup. Client runs $M_c$ locally and sends the features $z=M_c(x)$ to the server. The server predicts the intended attribute as $y^{\pub}=M_s(z)$. An adversary trains a separate model $M_a$ to predict the private attribute as $y^{\pri}=M_a(z)$.}    
%     \label{fig:scenario}
% \end{figure}
% % \end{SCfigure}

% In this paper, we propose~\sys{}, an alternative yet orthogonal method to label-specific obfuscation. Instead of censoring information that is utilized to extract known private labels, we hide information that is not used by the main model. The contributions of this paper are as follows:

In this paper, we propose an alternative solution in which, instead of removing the information that is related to a set of sensitive attributes, we discard the information that is not used by the server model to predict the target label. Our contributions are summarized in the following: 

\begin{itemize}[itemsep=1pt,leftmargin=0.5cm]%,leftmargin=0.5cm
\item We propose an {\it unsupervised obfuscation} mechanism, depicted in Figure~\ref{fig:scenario}, and formulate a general optimization problem for finding the obfuscated feature vector $z'=g(z)$. The formulation is based on minimization of mutual information between $z$ and $z'$, under a distortion constraint on model output $\|M_s(z)-M_s(z')\|$.
% We devise an obfuscation mechanism $z'=g(z)$ based on minimizing the mutual information between $z$ and $z'$. 
% Towards this, we formulate a minimization problem aiming to minimize the entropy of $z'$. 
% % We then connect mutual information to the concept of entropy, and show that data obfuscation can be achieved by minimizing the entropy of $z'$. 
% For the class of features following Gaussian priors, we obtain an analytical solution that minimizes the entropy of $z'$ while the distortion of model output  is bounded.} \aleksei{Maybe we shouldn't make such a bold claim, since we make additional assumptions and don't actually minimize the true entropy or mutual information, nor ensure bounded distortion...} 
We then devise a practical solution for a relaxation of the problem using the SVD of the first layer of $M_s$.

\item 
% \noindent{\bf Characterizing tradeoffs. } 
We perform extensive experiments on several datasets and show that our methods provide better tradeoffs between accuracy and obfuscation compared to existing approaches such as adversarial training, despite having no knowledge of the hidden attribute at training or inference phases. 
We also investigate the role of the edge computation and show that, with higher edge computation, the client obtains better obfuscation at the same target accuracy.

% We empirically show that our method reduces the accuracy of the adversary on sensitive attributes, at the cost of a small decrease in the target accuracy. We also show our method provides better tradeoffs between information obfuscation and accuracy compared to existing approaches. 
% discards more private information from features, which results in a less accurate adversary model. 
% Also, with the same edge computation (a given split layer), removing more components from the signal content provides better obfuscation at the cost of reduced accuracy. 

% \item We also investigate the edge computation efficiency, and show that a better accuracy/obfuscation tradeoff can be achieved with higher edge computation (more layers on the edge device). 
% \noindent{\bf Performing experiments. } 
% We perform extensive experiments on several datasets and show that our methods provide better tradeoffs between accuracy and information leakage compared to existing approaches, despite having no knowledge of the hidden attribute at training or inference times. 
% We show that our methods provide privacy as the private accuracy can be lowered down close to the random classifier, while maintaining the performance on the main task. 
\end{itemize}

\section{Problem Statement}
Let $M_c$ and $M_s$ be the client and server models, respectively, and $g$ be the obfuscation function. 
We consider a split inference setup, where the model $M_s \circ M_c$ is trained with a set of examples $\{x_i\}_{i=1}^N$ and their corresponding target labels $\{y^{\pub}_i\}_{i=1}^N$. 
At inference phase, clients run $g \circ M_c$ on their data and send $z'=g(z)$ to the server, where $z=M_c(x)$. 
The goal of the data obfuscation is to generate $z'$ such that it contains minimal information about the sensitive attributes, yet the predicted target label for $z'$ is similar to that of $z$, i.e., $M_s(z')\approx M_s(z)$. 
We consider the unsupervised data obfuscation setting, where the sensitive attributes are not available at training or inference phases, i.e., the obfuscation algorithm is required to be generic and remove information about any attribute that is irrelevant to the target label.

\subsection{Threat Model}

\noindent \textbf{Client model.} 
Upon receiving the service, the client decides on the best tradeoff of accuracy, data obfuscation, and computational efficiency, based on their desired level of information protection and also the computational capability of the edge device. Similar mechanisms are already in use in AI-on-the-edge applications. For example, in the application of unlocking smart phones with face recognition, the client can specify the required precision in face recognition, where a lower precision will provide faster authentication at the cost of lower security~\citep{Samsung}.

\noindent \textbf{Server model.} The server is assumed to be honest-but-curious. It cooperates in providing an inference mechanism with minimum information leakage to abide by law enforcement~\citep{gdpr} or to have competitive advantage in an open market. The server performs inference of the target attribute $y^\pub$, but might try to extract the sensitive information from the obfuscated  feature vector, $z'$, as well.

\noindent \textbf{Adversary model.} 
The adversary tries to infer sensitive attribute(s), $y^\pri$, from the obfuscated feature vector, $z'$. We consider a strong adversary with full knowledge of the client and server models, the training data and training algorithm, and the client's obfuscation algorithm and setting. 
The client and the server, however, do not know the adversary's algorithm and are not aware of the sensitive attributes that the adversary tries to infer.

% \noindent \textbf{Adversary.} To extract private information, an adversary with access to the features (could be the server itself) may train an adversary model shown as $M_a$ in Figure~\ref{fig:scenario}. The adversary has access to the public model $M_s\circ M_c$, the training data $\{x_i\}_{i=1}^N$, and their corresponding private labels $\{y^{\pri}_i\}_{i=1}^N$, and trains $M_a$ with $(z_i, y_i^\pri)$ where $z_i=M_c(x_i)$. Note that the server does not use $M_a$ in training $M_s\circ M_c$. In our experiments, we train $M_a$ as an aftermath process to evaluate the privacy of the trained $M_s\circ M_c$.

\section{Related Work}\label{sec:related}
Prior work has shown that representations learned by neural networks can be used to extract sensitive information~\citep{song2019overlearning} or even reconstruct the raw data~\citep{mahendran2015understanding}. Current methods for data obfuscation can be categorized as follows.

\noindent{\bf Cryptography-based Solutions.} 
%Since the server is not trusted, solutions based on public key encryption~(\citep{al2003certificateless}) are not applicable.
A class of public-key encryption algorithms protects the data in the transmission phase~\citep{al2003certificateless}, but cannot prevent data exposure to a curious server. Nevertheless, these methods can be used in conjunction with our approach to strengthen the defense against external adversaries. 
Another type of cryptographic methods allows running inference directly on the encrypted data~\citep{rathee2020cryptflow2} at the cost of significant communication and computation overhead. As an example, using the state-of-the-art cryptographic method, performing inference on a single ImageNet data takes about half an hour and requires $\sim32$GB data transmission~\citep{rathee2020cryptflow2}. We consider scenarios where the server provides service to millions of users (e.g., in Amazon Alexa or Apple Siri applications), and users expect low communication and fast response. Hence, classic solutions for secure function evaluation are not applicable to our scenario due to their high computational and communication cost. 

\noindent{\bf Noise Injection}.
In this method, the client sends a noisy feature $z'\sim h_\theta(z)$ instead of $z$ to the server, where $z'$ is drawn from a randomized mechanism $h_\theta(\cdot)$ parameterized by $\theta$ (e.g., a Gaussian distribution). 
A typical approach one could employ is differential privacy (DP)~\citep{dwork2006differential,dwork2006calibrating,dwork2014algorithmic}, which guarantees that the distribution of $z'$ does not differ too much for any two inputs $z_1$ and $z_2$. Using DP, however, can lead to a large loss of accuracy
% , especially in the local model that would be used in this context, due to considering a very strong adversary with arbitrary auxiliary knowledge
~\citep{kasiviswanathan2011can,ullman2018tight,kairouz2019advances}. 
% In general, differential privacy provides strong obfuscation but often performs. 
To maintain the utility of the model,~\citep{mireshghallah2020shredder} proposed to solve the following:
\begin{equation}\label{eq:related_mutual}
\small
    \underset{\theta}{\min}\ \  \mathcal{L}(M_s(z'), y^\pub) -  \gamma \mathcal{L}(M_a(z'), y^\pri)-\alpha |z'-z|,
\end{equation}
where the first and second terms denote the cross entropy loss for the server and adversary, respectively. 
In general, while noise addition improves the privacy, it has been shown to significantly reduce the accuracy~\citep{liu2019better,li2019deepobfuscator}.
% In general, noise-injection approaches have been shown to adversely impact accuracy 

% While noise addition can be highly beneficial for privacy, it has been shown to significantly reduce the accuracy~\citep{liu2019better,li2019deepobfuscator}. \hh{more references/elaboration here would help to make argument stronger.}

\noindent{\bf Information Bottleneck (IB)} is proposed to obfuscate the information related to a known set of sensitive attributes, $y^{\pri}$. Let $I(\cdot,\cdot)$ denote mutual information. The idea is to train $M_c$ that maximizes $I(z,y^{\pub})$ while minimizing $I(z,y^{\pri})$~\citep{osia2018deep,moyer2018invariant}. The optimization is formulated as follows:
% \begin{equation}\label{eq:related_mutual}
%     \underset{M_c}{\max}\quad \underset{{x,y^{\pub},y^{\pri}}}{\mathbb{E}}[I(M_c(x),y^{\pub})-\gamma I(M_c(x),y^{\pri}) - \beta I(M_c(x),x)].
% \end{equation}
\begin{equation}\label{eq:related_adversarial}
    \underset{M_c}{\max}\quad I(M_c(x), y^\pub) - \gamma I(M_c(x), y^\pri)-\beta I(M_c(x), x).
\end{equation}

\noindent{\bf Adversarial Training (AT)} is an effective method for obfuscating the information of a known set of sensitive attributes, while maintaining the accuracy on the target label. AT solves the following min-max optimization problem:
\begin{align}\label{eq:adversarial}
\underset{M_c, M_s}{\max}\ \underset{M_a}{\min}\  \underset{x,y^{\pub},y^{\pri}}{\mathbb{E}}[\gamma\mathcal{L}(y^{\pri}, M_a\circ M_c(x)) - \mathcal{L}(y^{\pub}, M_s\circ M_c(x))],
\end{align}
where $\mathcal{L}$ denotes the cross-entropy loss. 
The above objective can be achieved through an adversarial training method~\citep{edwards2015censoring,hamm2017minimax,xie2017controllable,li-etal-2018-towards,feutry2018learning,li2019deepobfuscator, huang2017context}. Upon convergence, the model $M_c$ generates $z$, using which $M_a$ cannot accurately estimate $y^{\pri}$, yet $M_s$ accurately predicts $y^{\pub}$. 
% The effectiveness of the above approach relies on the uniqueness of $M_a$; after adversarial training convergence, if $M_a$ is not unique, an alternative $M_a'$ can be trained on top of the (fixed) $M_c$ to extract $y^{pri}$, breaking the privacy of the system~(\citep{song2019overlearning}).

Existing obfuscation methods for split inference have several limitations. Except differential privacy which often significantly reduces the accuracy on $y^\pub$, the underlying assumption in the above methods is that a set of hidden attributes $y^\pri$ is provided at training time. In practice, however, it might not be feasible to foresee and identify all possible sensitive attributes and annotate the training data accordingly. 
It also contradicts deployment at-scale since whenever a new attribute needs to be protected, the client model has to be retrained and re-distributed to all edge devices that use the service. 
Current approaches also often provide a poor tradeoff between accuracy and preventing the information leakage. Moreover, the tradeoff of accuracy and obfuscation with the client-side computation is not well studied in the split learning framework.
In this paper, we characterize this tradeoff and propose to remove the content irrelevant to the main task, instead of obfuscating a predefined set of sensitive attributes. 
We empirically show that our method reduces the attack accuracy on hidden attributes, which are not known to the client or the server at training or inference times, at a small or no cost to the accuracy on the target label.

% In this paper, we characterize this tradeoff and propose an alternative approach of removing the content irrelevant to the main task, instead of obfuscating a pre-defined set of sensitive attributes. 
% We empirically show our method successfully reduces the attack accuracy on hidden attributes (not seen during training) at a small or no cost to accuracy.  

% Our notion of null-space content is defined regardless of the private attribute to be filtered out from $z$. The null-space identification can be performed after the original model $M_s\circ M_c$ is trained on the server. This approach allows fast adaptation which provides the customers of cloud-based inference with a tradeoff between computation load, accuracy, and privacy. 

% Informally, the null-content of $z$ is any information in $z$ that is not used by $M_s$ to perform the main task. By definition, cleansing the null-space content does not affect utility.
% We empirically show that removing null-space content from $z$ limits the adversary's ability in extracting private attributes.
\section{The Proposed Method}\label{sec:problem}

\subsection{Problem Formulation} 
% Let $z\in\mathbb{R}^n$ and $y^\pub$ be random variables denoting the raw feature vector and its corresponding label, respectively. 
Let $z\in\mathbb{R}^{n\times 1}$ be the feature vector and $y^\pub$ be the corresponding label. Our goal is to design the obfuscation function, $g$, such that $z'=g(z)$ contains the necessary and sufficient information about $y^\pub$. Specifically, we want to (i) minimize the information that $z'$ carries about $z$, while (ii) maintaining the utility of the model for predicting $y^\pub$ as much as possible. We formulate the problem as follows:
\begin{equation}\label{eq:main_obj}
    \underset{z'}{\min\ }I(z', z)\quad \mbox{ s.t. } \|M_s(z')-M_s(z)\|\leq \epsilon.
\end{equation}
% The mutual information $I(z', z)$ in the above formulation quantifies the amount of information that is obtained about $x$ given $x'$. 
% The term $I(z',y^\pub)-I(z',y^\pub)$ quantifies the reduction in the mutual dependence between the feature and the classification label. Hence, the optimization constraint implies that the distortion to model output caused by the obfuscation should be bounded. 
% where $I(\cdot,\cdot)$ denotes the mutual information. 
The objective function in~(\ref{eq:main_obj}) minimizes the mutual information of $z'$ and $z$, while bounding the distortion to the model output.
Note that our formulation is different from the information bottleneck~(\ref{eq:related_mutual}) proposed by~\citep{osia2018deep,moyer2018invariant} in that it does not use $y^\pri$ and, hence, is unsupervised with respect to hidden attributes.

% where $H(z')=-\int P(z')\log P(z') dz'$ denotes the entropy of the underlying random variable. By definition, the entropy encapsulates the expected value of uncertainty in the random variable. 
% In our case, the client knows the algorithm to convert $z\rightarrow z'$, hence $z'$ is deterministic given $z$ and we have $H(z'|z)=0$\footnote{the entropy of a deterministic variable is zero.}. 

Let $H(\cdot)$ denote the entropy function. We have:
\begin{align}
\nonumber I(z', z) = H(z')-H(z'|z) = H(z'),
\end{align}
where the second equality holds since the obfuscation function $g$ is a deterministic algorithm and thus $H(z'|z)=0$. 
The objective in~(\ref{eq:main_obj}) can therefore be written as:
\begin{equation}\label{eq:entropy_obj}
    \underset{z'}{\min\ }H(z')\quad \mbox{ s.t. } \|M_s(z')-M_s(z)\|\leq \epsilon.
\end{equation}
Intuitively, a small $H(z')$ indicates that, from the server's point of view, the incoming queries look similar. For example, in face recognition applications, only basic properties of face images are transmitted to the server and other irrelevant attributes that would make the images different in each query, such as the background or makeup, are obfuscated. 
% The constraint $H(y^\pub|z')-H(y^\pub|z)\leq \epsilon$ also indicates that the obfuscation should be done in a way that enables the server to classify $z'$ similar to $z$. 
% achieves a similar accuracy when using either $z$ or $z'$ as input. 
% In a more practical setting, given a model $M_s$ trained with $(z, y^\pub)$, the classification of obfuscated data $M_s(z')$ should be almost as accurate as $M_s(z)$. 

\subsection{Obfuscation for Linear Layers}\label{sec:lin_analysis}
Our goal is to develop a low-complexity obfuscation function, $g$, that solves~(\ref{eq:entropy_obj}) with respect to the server's model, $M_s$. 
Figure~\ref{fig:scenario} shows the block diagram of the method. In its general form, the function $g$ can be viewed as an auto-encoder (AE) network that is trained  with the objective of~(\ref{eq:entropy_obj}).
% to transform $z$ into $z'$ in such a way that the entropy of $z'$ is minimized, while the output of $M_s$ remains the same. 
Such a network would be, however, computationally complex to run on edge devices and defeats the purpose of sending activations to the server for low-complexity inference.

To address the computational complexity problem, we design the obfuscation function with respect to the first linear layer (a convolutional or fully-connected layer) of $M_s$. 
For linear models, the objective in~(\ref{eq:entropy_obj}) can be written as: 
\begin{equation}\label{eq:entropy_linear}
\underset{z'}{\min\ }H(z')\quad \mbox{ s.t. } \|Wz-Wz'\|\leq \epsilon,
\end{equation}
where $W\in\mathbb{R}^{m\times n}$ is the weight matrix. 
% The solution of~(\ref{eq:entropy_linear}) satisfies the constraint of~(\ref{eq:entropy_server}) \aleksei{[As far as I can tell that's only the case if the rest of the model is Lipschitz continuous with the Lipschitz constant $\leq 1$. In all other cases, this equivalence will only be true for $\epsilon=0$.]} but would have greater entropy compared to the solution of the general case \aleksei{[I don't see why would that be the case...]}. 
In the following, we present our analysis for the linear models.

\begin{definition}\label{def:W}
Let $W_{m\times n}=U_{m\times m}S_{m\times n}V_{n\times n}^T$ denote the singular value decomposition (SVD) of $W$. The columns of $V^T$ provide an orthonormal basis $\{v_k\}_{k=1}^{n}$. We write: 
\begin{equation}
z = \sum_{k=1}^{n}\alpha_k v_k=\sum_{k=1}^{n}z_k, \quad \alpha_k=v_k^T\cdot z.
\end{equation}
% where $<\cdot, \cdot>$ denotes vector dot product.
\end{definition}

% Let $V=\{v_k\}_{k=1}^{n}$ be a set of orthonormal basis for $z\in\mathbb{R}^n$, then we have:
% \begin{equation}\label{eq:basis}
%     z = \sum_{i=1}^{n}z_i,  \qquad z_i = \alpha_i v_i^T, \qquad \alpha_i = <v_i^T, z>,
% \end{equation}
% where the $<\cdot,\cdot>$ operator denotes inner-product.

We have the following Lemmas.

% The proof of Lemmas are provided in Appendix A.
% \begin{lemma}\label{lem:za}
% If $z\sim N(\mathbf{0}, \sigma^2\mathbf{I})$, then $\alpha_k$'s are independent random variables with $\alpha_k\sim N(0, \sigma^2)$.
% \end{lemma}

% \begin{lemma}\label{lem:entropy}
% For $z\sim N(\mathbf{0}, \sigma^2\mathbf{I})$, we have $H(z)=\sum_{k=1}^{n}H(\alpha_k)$.
% \end{lemma}

% \begin{lemma}\label{lem:var}
% For $\alpha_k\sim N(\mu, \sigma^2)$, the entropy $H(\alpha_k)$ decreases by decreasing the variance of $\alpha_k$.
% \end{lemma}
% \begin{proof}
% The entropy $H(\alpha_k)=\frac{1}{2}\ln(2\pi e \sigma^2)$ is monotonically increasing in $var(\alpha_k)=\sigma^2$. 
% % Since $H(\alpha_k)$ is monotone in $\sigma^2$, it can be reduced by reducing the variance of $\alpha_k$, i.e., by enforcing $\alpha_k$ to be close to zero.
% \end{proof}

% Lemma~\ref{lem:entropy}-\ref{lem:var} suggest that we can reduce the entropy $H(z')$ in Eq.~(\ref{eq:entropy_linear}) by reducing the variance of $\alpha_k$, i.e., the projection of $z$ onto orthonormal basis vectors. In the rest of our analysis, we aim to find out how this variance should be reduced such that $z'$ also satisfies the condition of Eq.~(\ref{eq:entropy_linear}). Specifically, we show how the entropy is minimized when the basis vectors $V=\{v_k\}_{k=1}^{n}$ are chosen from the SVD of $W$.

% \begin{lemma}\label{lem:distortion}
% Let $z=\sum_{k=1}^{n}\alpha_k v_k$ and $z'=\sum_{k=1}^{n}\alpha'_k v_k$. We have $\|W(z-z')\|=\sqrt{\sum_{k=1}^{m}(\alpha_k-\alpha'_k)^2s_k^2}$, where $s_k$ is the $k$-th singular value of $W$.
% \end{lemma}

\begin{lemma}\label{lem:za}
If $z\sim N(\mathbf{0}, \sigma^2\mathbf{I})$, then $\alpha_k$'s are independent random variables with $\alpha_k\sim N(0, \sigma^2)$.
\end{lemma}
\begin{proof}
Since $\alpha_k=v_k^T \cdot z$, it is Gaussian with the following mean and variance:
\begin{align}
\nonumber
% \begin{cases}
\mathbb{E}[\alpha_k]=v_k^T \cdot \mathbb{E}[z]=0,\qquad
\mathbb{E}[\alpha^2_k]=v_k^T \cdot \mathbb{E}[z z^T] \cdot v_k =\sigma^2 v_k^T \cdot v_k = \sigma^2.
% \end{cases}
\end{align}

Assume $i\neq j$. We have:
% \begin{equation}
% \nonumber \resizebox{.9\hsize}{!}{ cov(\alpha_i,\alpha_j)=\mathbb{E}[\alpha_i\cdot \alpha_j] - \mathbb{E}[\alpha_i]\cdot \mathbb{E}[\alpha_j] = v_i \cdot \mathbb{E}[z z^T] \cdot v_j^T = 0.}
% \end{equation}
\begin{align}
\nonumber cov(\alpha_i,\alpha_j)=\mathbb{E}[\alpha_i\cdot \alpha_j] - \mathbb{E}[\alpha_i]\cdot \mathbb{E}[\alpha_j] = v_i^T \cdot \mathbb{E}[z z^T] \cdot v_j = 0.
\end{align}
Also, $\beta=c_1\alpha_i + c_2\alpha_j = (c_1 v_i^T + c_2 v_j^T)\cdot z = v'^T z$. Therefore, since $cov(\alpha_i,\alpha_j)=0$ and $\beta$ is Gaussian for any $c_1$ and $c_2$, $\alpha_k$'s are independent random variables. 
\end{proof}

\begin{lemma}\label{lem:indep}
For $z\sim N(\mathbf{0}, \sigma^2\mathbf{I})$, we have $H(z_i|z_j)=H(z_i), \forall i\neq j$.
\end{lemma}
\begin{proof}
Assume $i\neq j$. We have: 
\begin{align}
\nonumber cov(z_i,z_j)&=\mathbb{E}[z_i\cdot z_j^T] - \mathbb{E}[z_i]\cdot \mathbb{E}[z_j]^T = \mathbb{E}[\alpha_i \alpha_j]v_i\cdot v_j^T - \mathbb{E}[\alpha_i]\mathbb{E}[\alpha_j]v_i\cdot v_j^T = \mathbf{0}.
\end{align}
Also, $x=c_1z_i + c_2z_j=c_1 \alpha_i v_i + c_2 \alpha_j v_j$. According to Lemma~\ref{lem:za}, $\alpha_k$'s are independent Gaussian random variables. Thus, $x$ is multivariate Gaussian for any $c_1$ and $c_2$. Hence, $z_k$'s are independent random vectors and $H(z_i|z_j)=H(z_i), \forall i\neq j$. 
\end{proof}

\begin{lemma}\label{lem:entropy}
For $z\sim N(\mathbf{0}, \sigma^2\mathbf{I})$, we have $H(z)=\sum_{k=1}^{n}H(\alpha_k)$.
\end{lemma}
\begin{proof}
Since $z_k=\alpha_k v_k = (v_k^T z) v_k$, we have $H(z_k|z)=0$. Lemma~\ref{lem:indep} also shows $H(z_i|z_j)=H(z_i), \forall i\neq j$. Hence:
\begin{equation}
\nonumber
     H(z, z_1, \dots, z_n)=H(z)+\sum\limits_{k=1}^{n}H(z_k|z, z_1, \dots, z_{k-1})=H(z).
\end{equation}
The left hand side can be also written as:
\begin{align}
\nonumber H(z, z_1, \dots, z_n)&=H(z_1)+H(z|z_1)+\sum\limits_{k=2}^{n}H(z_k|z, z_1, \dots, z_{k-1})\\
\nonumber &=H(z_1)+H(z|z_1)\\
\nonumber &=H(z_1)+H(z_2, \dots, z_n)\\
\nonumber &=H(z_1)+\sum\limits_{k=2}^{n}H(z_k|z_2,\dots, z_{k-1})\\
\nonumber &=\sum\limits_{k=1}^{n}H(z_k).
\end{align}
% \begin{equation}
% \nonumber
% \begin{array}{lll}
%      \quad H(z, z_1, \dots, z_n)&=H(z_1)+H(z|z_1)+\sum\limits_{k=2}^{n}H(z_k|z, z_1, \dots, z_{k-1})
%      &=H(z_1)+H(z|z_1)\\
%      =H(z_1)+H(z_2, \dots, z_n)
%      &=H(z_1)+\sum\limits_{k=2}^{n}H(z_k|z_2,\dots, z_{k-1})
%      &=\sum\limits_{k=1}^{n}H(z_k).
% \end{array}
% \end{equation}
Hence, $H(z)=\sum\limits_{k=1}^{n}H(z_k)$. We also have:
\begin{equation}
\nonumber 
\begin{array}{lll}
    H(z_k, \alpha_k) &= H(z_k)+H(\alpha_k|z_k) &= H(z_k)\\
    &=H(\alpha_k)+H(z_k|\alpha_k)&=H(\alpha_k).
\end{array}
\end{equation}
Therefore, $H(z_k)=H(\alpha_k)$, and $H(z)=\sum\limits_{k=1}^{n}H(\alpha_k)$.
\end{proof}

\begin{lemma}\label{lem:distortion}
Let $z=\sum_{k=1}^{n}\alpha_k v_k$ and $z'=\sum_{k=1}^{n}\alpha'_k v_k$. We have $\|W(z-z')\|=\sqrt{\sum_{k=1}^{m}(\alpha_k-\alpha'_k)^2s_k^2}$, where $s_k$ is the $k$-th singular value of $W$.
\end{lemma}
\begin{proof}
We have:
\begin{align}
\nonumber \ W(z-z') = USV(z-z')&= \sum\limits_{k=1}^{n}USV((\alpha_k-\alpha_k')v_k)\\
\nonumber &=\sum\limits_{k=1}^{n}US((\alpha_k-\alpha_k')\delta_k) \\
\nonumber &=\sum\limits_{k=1}^{n}U(s_k(\alpha_k-\alpha_k')\delta_k)\\
\nonumber &=\sum\limits_{k=1}^{n}s_k(\alpha_k-\alpha_k')U_k,
\end{align}
% \begin{equation}
% \nonumber 
% \begin{array}{lll}
%     \quad \ W(z-z') &= USV(z-z')&= \sum\limits_{k=1}^{n}USV((\alpha_k-\alpha_k')v_k)\\
%     =\sum\limits_{k=1}^{n}US((\alpha_k-\alpha_k')\delta_k)
%     &=U(s_k(\alpha_k-\alpha_k')\delta_k)&=\sum\limits_{k=1}^{n}s_k(\alpha_k-\alpha_k')U_k,
% \end{array}
% \end{equation}
where $\delta_k$ is a one-hot vector with its $k$-th element set to 1 and $U_k$ is the $k$-th column of $U$. Since $U_k$'s are orthonormal, we have $\|W(z-z')\|=\sqrt{\sum_{k=1}^{m}(\alpha_k-\alpha'_k)^2s_k^2}$.
\end{proof}

The following theorem provides the solution of the objective in (\ref{eq:entropy_linear}). 
\begin{theorem}\label{theorem:distortion}
Let $W=USV^T$ as defined in Definition~\ref{def:W}. Let $z=\sum_{k=1}^{n}\alpha_kv_k$ where $z\sim N(\mathbf{0}, \sigma^2\mathbf{I})$ and $\alpha_k$'s are sorted based on the singular values of $W$. The objective in (\ref{eq:entropy_linear}) is minimized by
$z'=\sum_{k=1}^m\alpha'_kv_k$ where
\begin{equation}
\alpha'_{k} = 
\begin{cases} 
\alpha_k &   k<m'\\ 
\alpha_k-\gamma \sign(\alpha_k) &  k=m'\\
0 &  k>m' 
\end{cases},
\end{equation}
where $m'=\underset{k}{\argmin}\ \epsilon_k\leq\epsilon$, $\epsilon_k=\sqrt{\sum_{i=k+1}^{m}\alpha_i^2s_i^2}$ and
$\gamma=\frac{\sqrt{\epsilon^2-\epsilon_{m'}^2}}{s_{m'}}$.
% $\gamma=\min(\alpha_{m'}, \frac{\sqrt{\epsilon-\epsilon_k}}{s_{m'}})$.
\end{theorem}

\begin{proof}
Using Lemmas~\ref{lem:za} and~\ref{lem:entropy}, we have $H(z)=\sum_{k=1}^{n}H(\alpha_k)$, where $\alpha_k\sim N(0, \sigma^2)$. 
We have $H(\alpha_k)=\frac{1}{2}\ln(2\pi e \sigma^2)$ which is monotonically increasing in $var(\alpha_k)=\sigma^2$. Hence, $H(\alpha_k)$ can be reduced by suppressing the variance of $\alpha_k$, i.e., making $\alpha_k$'s closer to zero. 

Given a distortion budget, $\epsilon$, the question now is which $\alpha_k$ should be modified and by how much. From the entropy perspective, based on the assumptions above and Lemma~\ref{lem:entropy}, reducing the variance of each $\alpha_k$ reduces the entropy by the same amount for all $k$. Lemma~\ref{lem:distortion}, however, states that modifying $\alpha_k$ by $\gamma$ causes a distortion of $|\gamma|s_k$, where $s_k$ is the $k$-th singular value of $W$. Since smaller $s_k$'s cause smaller distortion, the solution is achieved by sorting the singular values and then modifying the $\alpha_k$'s corresponding to the smaller singular values towards zero one at a time until the budget $\epsilon$ is exhausted. 

The following provides the solution more specifically. If $m<n$ in the weight matrix, the last $n-m$ coefficients do not contribute to $W\cdot z$ and thus can be set to zero without causing any distortion. Now, assume the coefficients in range of $m'+1$ to $m$ are to be set to zero. The total distortion will be $\epsilon_{m'}=\sqrt{\sum_{i=m'+1}^{m}s_i^2\alpha_i^2}$. 
Also, the distortion caused by modifying $\alpha_{m'}$ by $\gamma$ is $s_{m'}|\gamma|$, which we will set to be equal to the remaining distortion, $\sqrt{\epsilon^2-\epsilon_{m'}^2}$, i.e., $\gamma=\frac{\sqrt{\epsilon^2-\epsilon_{m'}^2}}{s_{m'}}$. This completes the proof.
\end{proof}

% We first define the concepts of signal and null contents based on the analysis of Section~\ref{sec:lin_analysis}. 
\begin{definition}\label{def:signal_content} 
The signal content of $z$ with respect to a matrix $W$, or simply the signal content of $z$, is the solution to~(\ref{eq:entropy_linear}) with $\epsilon=0$. It is denoted by $z_{\mathcal{S}}$ and defined as follows:
\begin{equation}\label{eq:zS}
z_{\mathcal{S}} =  \sum\limits_{k=1}^{m}\alpha_kv_k.
\end{equation}
The remaining $n-m$ components of $z$ are called the null content defined as follows:
\begin{equation}\label{eq:zN}
z_{\mathcal{N}}=z-z_{\mathcal{S}}=\sum_{k=m+1}^{n}\alpha_kv_k.
\end{equation}
\end{definition}

The signal content is the information that is kept after multiplying $z$ by $W$, and the null content is the discarded information. By setting $z'=z_{\mathcal{S}}$, the client reduces the entropy without introducing any distortion to the output of $M_s$. We call this method \textit{distortion-free obfuscation} herein. 
The entropy can be further reduced by removing components from the signal content as well, for which the optimal way for a desired distortion $\epsilon$ is determined by Theorem~\ref{theorem:distortion}. We call this method \textit{distortion-bounded obfuscation} in the remainder of the paper.

% The optimal way of doing so is determined by Theorem~\ref{theorem:distortion} for a desired distortion, $\epsilon$. 

% \vspace{-0.25cm}
\subsection{The Proposed Obfuscation Method}
In the following, we present our framework for unsupervised data obfuscation in the split inference setup. 

\noindent{\bf Training.} 
The server trains the model $M_s\circ M_c$ with inputs $\{x_i\}_{i=1}^N$ and target labels $\{y^\pub_i\}_{i=1}^{N}$, where, at each epoch, various fractions of the signal content of different layers are removed (one layer at a time), so that the model becomes robust to removing the components of the signal content. 
The model is also trained to generate feature vectors, $z$, with uncorrelated Gaussian activations as specified in Theorem~\ref{theorem:distortion}. 
To learn models with decorrelated activations, we used the penalty term proposed in~\citep{cogswell2015reducing} as $L_{deCov}=\frac{1}{2}(\|C\|_F^2-\|\diag(C)\|_2^2)$, where $C$ is the covariances between activation pairs, $\|\cdot\|_F$ is the Frobenius norm, and the $\diag(\cdot)$ operator extracts the main diagonal of a matrix into a vector. 
Additionally, the distribution of $z$ is forced to be close to Gaussian using the VAE approach~\citep{kingma2013auto}, i.e., by learning to generate $z$ from a variational distribution with a Gaussian prior.

% The output of DNN layers can have arbitrary distribution that is not known at test time; however, we have to assume that the data follows some distribution at test time in order to devise an obfuscation algorithm. Our assumption is that each element of vector $z$ follows a Gaussian distribution with zero mean and standard deviation of $\sigma$. This is a viable assumption since features are often processed through batch-normalization layers. 

\noindent{\bf Inference.} 
Upon providing the service, the server also provides a profile of the average reduction in target accuracy by removing a given fraction of the signal content for each split layer. 
The client first decides on the number of layers to be run locally on the edge device (determined based on the compute capacity) and then on the fraction of the signal content to maintain (determined based on the desired accuracy). 
For inference, the client computes the obfuscated feature vector $z'=g(M_c(x))$ and sends it to the server. The server then performs the rest of the computation and obtains $\hat{y}^\pub=M_s(z')$.  

Our framework provides a tradeoff between accuracy, obfuscation, and computational efficiency. 
Specifically, by running more layers locally (more edge computation), the client can achieve a better accuracy-obfuscation tradeoff, i.e., the same obfuscation can be obtained by discarding a smaller fraction of the signal content. 
Moreover, for a given split layer, the client can adjust the fraction of the signal content to be removed in order to obtain a desired tradeoff between accuracy and obfuscation. In Section~\ref{sec:experiment}, we provide empirical validations for the aforesaid tradeoffs.     

\noindent{\bf Computational and Communication cost.} Performing obfuscation requires the client to compute $m'$ coefficients $\{\alpha'_k\}_{k=1}^{m'}$ on the edge device, where the overhead of computing each $\alpha'_k$ is equivalent to $\frac{1}{m}$-th of total computation in the first layer of $M_s$. Therefore, the client performs an extra computation equivalent to $\frac{m'}{m}\times$ the first layer of $M_s$, where $m'\ll m$ in practice. Note that the client is not required to recover $z'=\sum_{k=1}^{m'}\alpha'_kv_k$, and can send only $\{\alpha'_k\}_{k=1}^{m'}$ to the server, who has the $\{v_k\}$ basis and can compute $z'$ accordingly. Therefore, our obfuscation method reduces the communication cost by a factor of $\frac{m'}{n}$ compared to the case that the raw feature vector $z$ is sent to the server.

\section{Experimental Results}\label{sec:experiment}
% \vspace{-0.25cm}

%%%%%%%%%%%%%%%%%%%%%%%%%%%%%%%%%%%%%%%%%%

\subsection{Experiment Setup}

\noindent{\bf Model architecture and training settings.} We present the experimental results on an architecture used in prior work~\citep{song2019overlearning}, shown in Table~\ref{tab:architecture}. The  adversary model $M_a$ has the same architecture as the server model $M_s$. 
We train the models for $50$ epochs using Adam optimizer with an initial learning rate of $0.001$, which we drop by a factor of $10$ after $20$ and $40$ epochs.

\noindent{\bf Datasets.} We perform our experiments on four visual datasets described below. Table~\ref{tab:baseline-accs} lists the target and hidden attributes of the datasets used. 
\vspace{-0.3cm}
\begin{itemize}[itemsep=2pt,leftmargin=*]
    \item {\bf EMNIST}~\citep{cohen2017emnist} is an extended version of the MNIST dataset where the labels are augmented with writer IDs. 
    % To have a uniform distribution of writer identities, 
    We select $13000$ samples from EMNIST written by $100$ writers with $130$ examples per writer. We then split this dataset into $10000$, $1500$, and $1500$ for training, validation, and testing. We use the digit and writer ID as the target and the hidden attributes, respectively. 
    
    \item {\bf FaceScrub}~\citep{facescrub2,facescrub} is a dataset of celebrity faces labeled with gender and identity. We use gender as the target and identity as the hidden attribute. In experiments, we crop images using the face region bounding boxes specified in the annotations and resize them to $50\times 50$.
    
    \item {\bf UTKFace}~\citep{utkface} is a dataset of face images labeled with gender and race, which we treat as the target and the hidden attributes, respectively. The face region bounding boxes are cropped and resized to $50\times 50$.
    
    \item {\bf CelebA}~\citep{celeba} is a dataset of celebrity images. Each image is labeled with 40 binary attributes, out of which, we select \emph{Smiling} as the target attribute and \{\emph{Male}, \emph{Heavy\_Makeup}, \emph{High\_Cheekbones}, \emph{Mouth\_Slightly\_Open}, \emph{Wearing\_Lipstick, Attractive}\} as hidden attributes. These attributes have near balanced distribution of positive and negative examples. 
    In experiments, we crop images to the face region and resize them to $73\times 60$.
\end{itemize}

\begin{table}[t]
\centering\small
\caption{\small Network Architecture. Each row shows a  split layer, i.e., 
for the split layer $i$, the input of that layer is sent to the server.}\label{tab:architecture}
% \vspace{-0.1cm}
%\resizebox{\columnwidth}{!}{
\begin{tabular}{cl}
\toprule
{\textbf{Layer}}  & \textbf{Architecture} \\ 
\midrule
{\small \textbf{1}}  & {\small CONV($3,16$), ReLU, Maxpool($2\times 2$), Batchnorm} \\
{\small \textbf{2}}  & {\small CONV($3,32$), ReLU, Maxpool($2\times 2$), Batchnorm} \\
{\small \textbf{3}}  & {\small CONV($3,64$), ReLU, Maxpool($2\times 2$), Batchnorm} \\
{\small \textbf{4}}  & {\small FC(128), ReLU, Batchnorm} \\
{\small \textbf{5}}  & {\small FC(64), ReLU, Batchnorm} \\
{\small \textbf{6}}  & {\small FC($n_{\mbox{\tiny classes}}$), Softmax} \\
\bottomrule
\end{tabular}
%}
% \vspace{-0.25cm}
\end{table}

\begin{table}[t]
\centering
\caption{\small Target and hidden attributes of the datasets used.}\label{tab:baseline-accs}
\begin{tabular}{ccccc}
\hline
\multirow{2}{*}{\textbf{Dataset}} & \multicolumn{2}{c}{\textbf{Target Attribute ($y^\pub$)}} & \multicolumn{2}{c}{\textbf{Hidden Attribute ($y^\pri$)}} \\ \cline{2-5} 
                                  & \textbf{Attribute}          & \textbf{No. Classes}   & \textbf{Attribute}       & \textbf{No. Classes}      \\ \hline
EMNIST                            & digit                       & 10                     & writer ID                & 100                       \\ \hline
UTKFace                           & gender                      & 2                      & race                     & 5                         \\ \hline
FaceScrub                         & gender                      & 2                      & identity                 & 530                       \\ \hline
\multirow{6}{*}{CelebA}           & \multirow{6}{*}{smiling}    & 2                      & gender                   & 2                         \\
                                  &                             & 2                      & makeup                   & 2                         \\
                                  &                             & 2                      & cheekbones               & 2                         \\
                                  &                             & 2                      & mouth-open               & 2                         \\
                                  &                             & 2                      & lipstick                 & 2                         \\
                                  &                             & 2                      & attractive               & 2                         \\ \hline
\end{tabular}
\vspace{-0.5cm}
\end{table}

\vspace{-0.3cm}
\noindent{\bf Measuring obfuscation.} 
Several methods have been proposed to measure the information leakage of intermediate feature vectors in neural networks. One approach is computing the mutual information between the query $x$ and the feature vector $z'$~\citep{kraskov2004estimating}. In practice, measuring the mutual information is not tractable for high-dimensional random variables, unless certain assumptions are made about the probability distribution of the random variables of interest. A more practical approach computes the reconstruction error, $\|\widetilde{x}-x\|$, where $\widetilde{x}$ is estimated using the feature vector~\citep{mahendran2015understanding}. Finally, attribute leakage can be defined based on the accuracy of an adversary model that predicts the hidden label from intermediate features. 

\begin{figure}[ht]
    \centering
    % \vspace{-0.1cm}
    \includegraphics[width=0.5\linewidth]{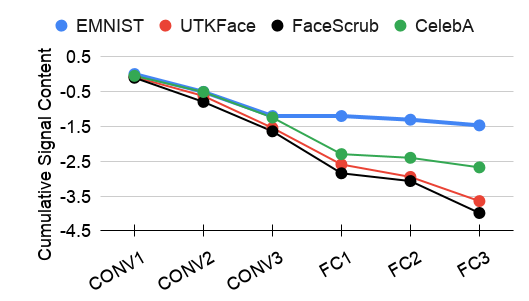}
    % \vspace{-0.1cm}
    \caption{\small Cumulative preserved signal content at different layers of $M_s\circ M_c$. 
    The model acts as an obfuscator, gradually removing the content irrelevant to the target label from one layer to the next.  }
    \label{fig:baseline}
    \vspace{-0.5cm}
\end{figure}

% In our threat model, the server only knows $y^\pub$ and can only use this information to train $M_s\circ M_c$. 
% At the inference phase, the client computes $z=M_c(x)$ locally and queries the server, which computes $y^{pub}=M_s(z)$.

% In this paper, our goal is to minimize the leakage of hidden attributes without the knowledge of those attributes. 
In this paper, we follow the approach of predicting hidden attributes using an adversary model. 
Assume that each example $\{x_i\}_{i=1}^N$ has a target label $\{y^\pub_i\}_{i=1}^{N}$ and a hidden label $\{y^{\pri}_i\}_{i=1}^N$. 
The adversary trains the model $M_a$ with $(z'_i, y_i^\pri)$, where $z'_i=g(M_c(x_i))$ is the same feature vector that the server also receives to do inference of the target label.
Note that the training of $M_a$ is used as a post hoc process to evaluate the leakage of sensitive attributes and does not influence the client or server's processing. 
We refer to the accuracy of $M_s \circ g \circ M_c$ on $y^{\pub}$ as \emph{target accuracy} and the accuracy of $M_a \circ g \circ M_c$ on $y^{\pri}$ as \emph{attack accuracy}.

% \noindent{\bf Client, server, and adversary assumptions. } We assume that the server trains $M_s\circ M_c$ using only training data and public labels.
% Note that the server is willing to cooperate in providing a privacy-preserving model that only extracts intended attributes. This is a viable scenario as providing privacy gives competitive advantage to different companies providing a similar service. 
% After training, the server and client may agree to split the inference at different layers. The adversary utilizes the output of $M_c$ to infer private labels. 

% \noindent{\bf Adversary capabilities.} We use the same architecture for adversary's model $M_a$ as the server model $M_s$. 
% The model $M_a$ is trained using the features extracted by $M_c$ and the associated private labels. We also assume that the adversary knows the parameters of $M_c$ and $M_s$. %, which makes him capable of extracting signal space content from noisy features. 
% The adversary cannot change the parameters of $M_c$ or $M_s$.

\begin{figure}[ht]
    \centering
    % \vspace{-0.1cm}
    \includegraphics[width=\textwidth]{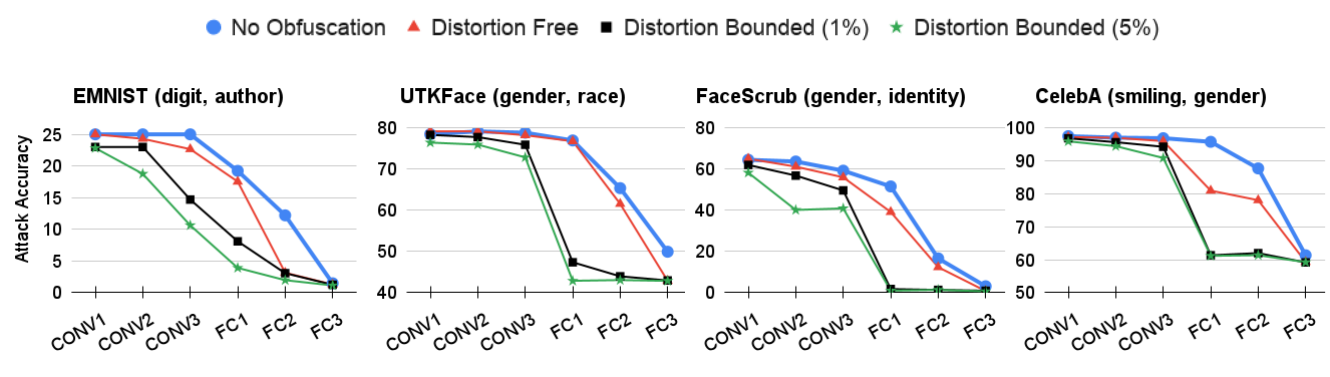}
    \vspace{-0.5cm}
    \caption{\small Attack accuracy versus the split layer for each dataset. The target/attack attributes are mentioned in the title of each sub-figure. Each colored curve corresponds to a different obfuscation method. The horizontal and vertical axes represent split layer and attack accuracy, respectively. It is desired to have a lower attack accuracy at earlier layers. The blue curves show the attack accuracy without obfuscation; the red curve represents our distortion-free obfuscation (with no loss of target accuracy); The black and green curves correspond to $1\%$ and $5\%$ of reduction in target accuracy, respectively. We observe two trends: (1)~by performing more computation on the edge device (e.g., splitting at FC1 layer instead of CONV1 layer), one can achieve a lower attack accuracy. (2)~for a certain split layer, our distortion free obfuscation reduces the attack accuracy of the baseline without reducing target accuracy. Our distortion-bounded accuracy provides a tradeoff between target accuracy and target accuracy. In this figure, for instance, we can see that the attack accuracy of the green curve ($5\%$ distortion) is lower than the black curve ($1\%$ distortion).  
    % The effect of the split layer and our proposed methods for filtering features vector on private accuracy. The Y-axis in each figure covers the range between random guess and the baseline accuracy. 
    % Removing the null content reduces the private accuracy without affecting the public accuracy. Also, removing part of the signal content further reduces the private accuracy, in some cases to the level of random guess, while slightly reducing the public accuracy.
    % The total number of features is $m=16$. 
    % For a given split layer, the number of preserved features ($m'$) can be tuned so as to achieve a desired tradeoff between utility (higher public accuracy) and privacy (lower private accuracy). Also, better tradeoffs can be achieved when the network is split from deeper layers.
    }\label{fig:private_acc_layers}
    % \vspace{-0.25cm}
\end{figure}

\subsection{Evaluations}\label{sec:content}

% \hh{in some cases (particularly EMNIST in Figure 5), public accuracy seems to take a larger hit than the private one. This can happen and we cannot control over}

% \hh{we do comparisons with AT because it provides the best tradeoff of accuracy and obfuscation compared to others, including adding noise and IB}

% \hh{More explanation is need about adversarial training in the evaluation section. 
% My understanding from the current text is that it is only trained to obfuscate one attribute (gender), but what if we train it to obfuscate more attributes? }

% \noindent{\bf Baselines.} The baseline test accuracy results are provided in Table~\ref{tab:baseline-accs}. Note that the accuracy on hidden attributes could be higher or lower than the target accuracy depending on the difficulty of the tasks. We do not compare these two quantities. Instead, we compare different methods in how much they reduce the attack accuracy, while maintaining the target accuracy.

\noindent{\bf Cumulative signal content.} 
We start our analysis by computing the norm of the null and signal contents in every layer of $M = M_s \circ M_c$. At each layer, the null content $z_\mathcal{N}$ is discarded and the signal content $z_\mathcal{S}$ is passed through the next layer. We compute the normalized amount of the information passed from the $i$-th layer to the next as $C_{\mathcal{S}}(z^{(i)})=\log(\frac{\|z_{\mathcal{S}}^{(i)}\|_2^2}{\|z^{(i)}\|_2^2})$, where $z^{(i)}$ and $z_{\mathcal{S}}^{(i)}$ are the activation vector and its signal content at the $i$-th layer, respectively.
% The normalized signal and null contents of $z$ are defined as $C_{\mathcal{S}}(z)=\frac{||z_{\mathcal{S}}||_2^2}{||z||_2^2}$ and $C_{\mathcal{N}}(z)=\frac{||z_{\mathcal{N}}||_2^2}{||z||_2^2}$, respectively.
% Recall that the signal-space and null-space content ($C_{\mathcal{S}}(z)$ and $C_{\mathcal{N}}(z)$) are in range $[0,1)$ defined in Section~\ref{sec:training}.
% \edit{For CONV layers, we compute $C_{\mathcal{N}}(z)$ for every $z$ that is obtained from a sliding window, then report the average of $C_{\mathcal{N}}(z)$ for all windows.} 
Figure~\ref{fig:baseline} shows the cumulative amount of the signal content preserved up to the $i$-th layer, computed as $\sum_{j=1}^{i}C_{\mathcal{S}}(z^{(j)})$. 
% As can be seen, the main network inherently discards information irrelevant to $y^\pub$ from one layer to the next, until the prediction is made at the last layer. 
The plot suggests that the model gradually removes the content irrelevant to the target label from one layer to the next, thus acting as an obfuscator.

% and $C_{\mathcal{S}}(\cdot)$ is presented in Definition~\ref{def:normalized_content}.
% signal and null contents for all layers across the evaluated datasets. As seen, features contain null-space content that is filtered out when the next layer is applied. 
% As the feature vector propagates through network layers, more content is gradually removed from $z$ until the model outputs the prediction on the main task. 

\begin{figure*}[t]
% \vspace{-0.1cm}
    \centering
    \includegraphics[width=\textwidth]{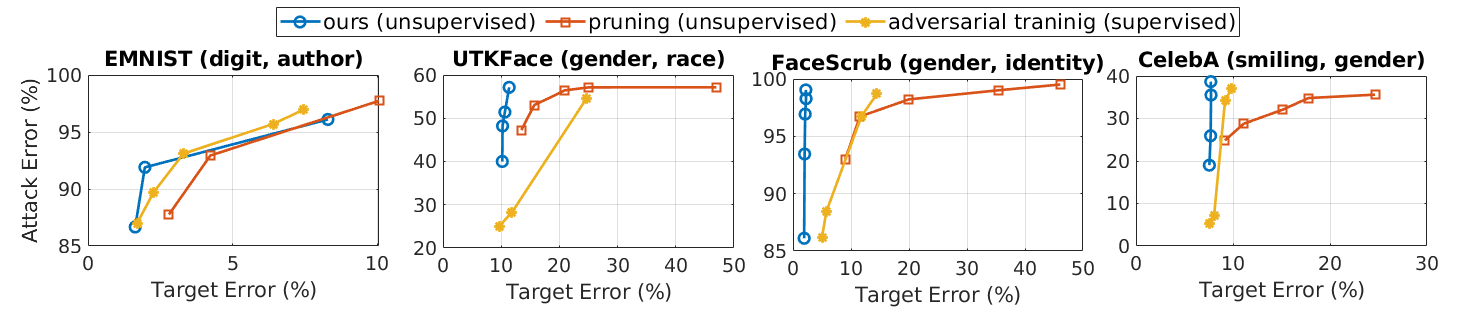}
    % \vspace{-0.4cm}
    \caption{\small Comparison between our method, feature pruning, and adversarial training when splitting the model at the input of the FC1 layer. The target/attack attributes are mentioned in the title of each sub-figure. The horizontal and vertical axes show the classification \textbf{error} on $y^\pub$ and $y^\pri$, respectively. To extract the points of the blue curve, we obfuscate the features at the split layer with different distortion rates. For pruning, the curves are obtained by pruning the features at the split layer with different pruning ratios. For adversarial training, each point on the curve represents an adversarially trained network with different $\gamma$ parameter. Our goal is to have low error on $y^\pub$ and high error on $y^\pri$, hence, points to the upper left corner of the figures are desirable. In most cases, our method outperforms both adversarial training and pruning.}
    \label{fig:compare_both}
    % \vspace{-0.2cm}
\end{figure*}

% \begin{figure}
%     \centering
%     \includegraphics[width=1\columnwidth]{Figures/Example-Tradeoff.png}
%     \caption{\textcolor{red}{\small Tradeoff between utility and obfuscation for the UTKFace dataset, when the network is split at the input of the FC1 layer. As we increase the number of preserved features ($m'$), utility is increased (lower target error) but obfuscation is decreased (lower attack error).}}
%     \label{fig:ex_m_prime}
% \end{figure}

% \begin{figure*}[t]
% % \vspace{-0.1cm}
%     \centering
%     \includegraphics[width=0.95\textwidth]{Figures/comparison_pruning.png}
%     % \vspace{-0.4cm}
%     \caption{\small Comparison between our method and feature pruning. The horizontal and vertical axes show the classification \textbf{error} on $y^\pub$ and $y^\pri$, respectively. Our goal is to have low error on $y^\pub$ and high error on $y^\pri$. In most cases, our method significantly outperforms feature pruning, i.e., with the same utility (target error), our method achieves better obfuscation (higher attack error). For CelebA, the sensitive attribute in this study is ``gender''.}
%     \label{fig:pub-priv-tradeoff}
%     % \vspace{-0.5cm}
% \end{figure*}

\noindent{\bf Investigating tradeoffs.} 
% We now evaluate the information leakage based on the attack accuracy. 
% The information leakage is considered high if the adversary can train a model $M_a$ that accurately estimates $y^\pri$. 
% In the remainder, we use the adversary (attack) accuracy as a proxy to information leakage.
% As such, it is safe to assume that the features become more related to the public task as we move towards the last layer. We next test this hypothesis by cutting the network at different layers and extracting private attributes from the features.
Figure~\ref{fig:private_acc_layers} shows the attack accuracy versus the split layer for four settings: (1)~without obfuscation, (2)~with distortion-free obfuscation (i.e., when $m'=m$), and (3),(4)~with distortion-bounded obfuscation such that the drop in target accuracy is at most $1\%$ and $5\%$, respectively. The results illustrate the tradeoffs between edge computation, obfuscation, and target accuracy, described in the following.
\begin{itemize}[itemsep=2pt,leftmargin=0.5cm]%[leftmargin=*]
    \item In all four cases, the attack accuracy decreases as the network is split at deeper layers. This observation indicates that more edge computation results in lower attack accuracy at the same target accuracy. 
    
    \item For each split layer, the attack accuracy of distortion-free obfuscation is less than that of the baseline (no obfuscation). This observation shows that even without decreasing the target accuracy, the feature vector can be modified to obtain a better obfuscation. 
    
    \item For each split layer, the distortion-bounded obfuscation further reduces the attack accuracy at the cost of a small reduction in target accuracy. In the example analysis of Figure~\ref{fig:private_acc_layers}, we include the attack accuracy when the target accuracy is dropped by $1\%$ and $5\%$. As seen, by losing more target accuracy, the attack accuracy can be further reduced. The tradeoff between attack and target accuracy is controlled by the the number of preserved features, $m'$.
    
    % \item For the same split layer (same edge compute), the number of preserved features, $m'$, provides a tradeoff between the target accuracy and obfuscation. We show an example of this effect in Figure~\ref{fig:ex_m_prime}.
\end{itemize}

\noindent{\bf Comparison to prior work.} We compare our obfuscation method to two general categories of prior work:

\begin{itemize}
    \item \textbf{Supervised Obfuscation.} Among existing supervised obfuscation methods introduced in Section~\ref{sec:related}, adversarial training generally provides the best tradeoff between accuracy and obfuscation. Therefore, we use adversarial training as a natural comparison baseline. We implemented the adversarial training framework proposed by~\citep{feutry2018learning} and trained the models in multiple settings with different $\gamma$ parameters (Eq.~\ref{eq:adversarial}) in range of $[0.1, 1]$ to achieve the best performance. 
    \item \textbf{Unsupervised Obfuscation.}  Similar to our approach, pruning network weights eliminates features that do not contribute to the classification done by $M_s$. Since pruning does not require access to the sensitive labels, it is a natural unsupervised obfuscation baseline. We adopt the pruning algorithm proposed by~\citep{li2016pruning} which works based on the $L_1$ norm of the columns of the following layer's weight matrix. After pruning, we fine-tune $M_s$ to improve its accuracy.
\end{itemize}

Figure~\ref{fig:compare_both} compares the tradeoffs achieved by our method with those achieved by (supervised) adversarial training and (unsupervised) pruning. In this example study, we split the network at the middle layer, i.e., the input of the $FC1$ layer. Since adversarial training specifically trains the model to obfuscate the sensitive attribute, it achieves a better tradeoff than the unsupervised pruning method. Although our method is unsupervised too, it outperforms the supervised adversarial training in most cases. The superior performance of our method is rooted in minimizing the distortion to the target accuracy while maximizing the obfuscated ``unrelated'' information.

\begin{figure}[t]
    % \vspace{-0.1cm}
    \centering
    \includegraphics[width=0.6\columnwidth]{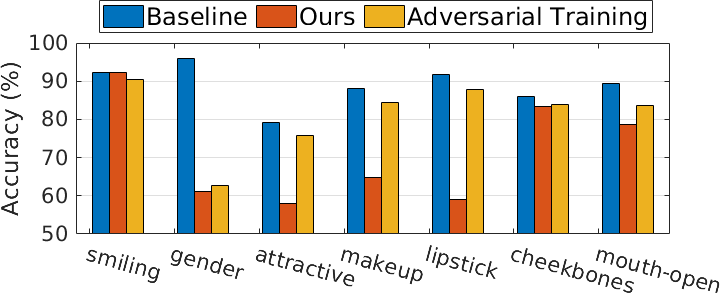}
    % \vspace{-0.2cm}
    \caption{\small Accuracy on the target (\emph{Smiling}) and hidden attributes. 
    Our method obfuscates the feature vector without the knowledge of hidden attributes at training or inference times. Adversarial training (AT) method 
    maximizes the accuracy on \emph{Smiling}, while minimizing the accuracy on \emph{Gender}.
    As seen, AT successfully reduces accuracy on \emph{Gender} attribute but, unlike our method, fails to obfuscate information of other attributes. This highlights the applicability of our method in practical settings as a generic obfuscator compared to specialized techniques such as AT. }
    \label{fig:multi-private}
    \vspace{-0.5cm}
\end{figure}

\noindent{\bf Experiment with multiple private attributes.} 
To highlight the independence of our method from private attributes,
we also do the experiments with multiple (unseen) hidden labels. Specifically, we consider the CelebA model trained to detect \emph{Smiling} and evaluate two methods, (1) our method: we keep only $m'=1$ component from the signal content of feature vector and then train one separate adversary model per hidden attribute, and (2) adversarial training: we first adversarially train an $M_c$ model to obfuscate \emph{Gender}, and then train one separate adversary model to predict each hidden attribute.

For both of the above methods, the network is split at the input of the $FC1$ layer. As shown in Figure~\ref{fig:multi-private}, our method outperforms adversarial training in both the target and attack accuracy. Specifically, our method results in a significantly lower attack accuracy on all hidden attributes compared to the baseline attack accuracy. The only exceptions are \emph{High\_Cheekbones} and \emph{Mouth\_Open} attributes, which highly correlate with the target attribute (a smiling person is likely to have high cheekbones and open mouth). The correlation between target and hidden attributes causes the signal content of the server and adversary models to have large overlaps and, hence, results in high attack accuracy. 
Also, as seen, the adversarially trained model successfully hides the information that it has been trained to obfuscate (\emph{Gender}). The model, however, fails to remove information of other attributes such as \emph{Makeup} or \emph{Lipstick}. The results highlight the importance of the generic unsupervised obfuscation in scenarios where the sensitive attributes are not known. In such cases, unlike supervised obfuscation methods, our method successfully reduces the information leakage.

\section{Conclusion}
We proposed an obfuscation method for split edge-server inference of neural networks. 
% In this setting, the edge device runs multiple layers locally and then obfuscates the intermediate feature vector before sending it to the server to execute the rest of the model. 
% The obfuscation is done by removing the information that is not relevant to the main task. 
We formulated the problem as an optimization problem based on minimizing the mutual information between the obfuscated and original features, under a distortion constraint on the model output. We derived an analytic solution for the class of linear operations on feature vectors. The obfuscation method is unsupervised with respect to sensitive attributes, i.e., it does not require the knowledge of sensitive attributes at training or inference phases. By measuring the information leakage using an adversary model, we empirically supported the effectiveness of our method when applied to models trained on various datasets. We also showed that our method outperforms existing techniques by achieving better tradeoffs between accuracy and obfuscation.

% ubfuscated and developed two methods of removing the content of the feature vector in the null space of the following linear layer and also removing the low-energy content of the remaining signal. 
% We showed that, unlike existing methods, our methods improve privacy without requiring the knowledge of private attributes at training or inference times. 

% We empirically showed that the edge device can enhance privacy by filtering out the null-content and low-energy parts of the signal-content. Unlike contemporary private split inference methods that obfuscate information relevant to a black-list of private attributes, our method obfuscates information that is not used by the model to extract the public label. As such, compared to prior art, our method can achieve better privacy with regards to attributes that were not foreseen during model training.

% We demonstrated that running more layers on the edge device boosts the privacy; this is due to the fact that each layer discards (private) information that is not used by the server model towards extracting the public attribute. To analyse the flow of information through network layers, we devise notions of null-content and signal-content. We showed that only the signal-content is required to infer the public attribute, and that the null-content contains information about private attributes. 

% \newpage
\bibliography{example_paper}
\bibliographystyle{plainnat}
% \newpage
% \input{7_Appendix(ICML)}

\end{document}